\documentclass{article}

\usepackage{microtype}
\usepackage{graphicx}
\usepackage{subcaption}
\usepackage{booktabs}
\usepackage{hyperref}

\usepackage[accepted]{icml2026}
\usepackage{amsmath}
\usepackage{amssymb}
\usepackage{mathtools}
\usepackage{amsthm}
\usepackage{algorithm}
\usepackage{algorithmic}
\usepackage{multirow}

\theoremstyle{plain}
\newtheorem{theorem}{Theorem}[section]
\newtheorem{proposition}[theorem]{Proposition}

\theoremstyle{definition}

\theoremstyle{remark}

\icmltitlerunning{UniRA: Unified Representation Alignment and Auto-Encoding}

\begin{document}

\twocolumn[
  \icmltitle{UniRA: Unified Representation Alignment and Auto-Encoding \\ for Efficient Diffusion Generation}

  \begin{icmlauthorlist}
    \icmlauthor{An Lanji}{uestc}
    \icmlauthor{Dawei Liu}{uestc}
    \icmlauthor{Jin Li}{uestc}
    \icmlauthor{Haoran Xu}{uestc}
    \icmlauthor{Mei Chen}{uestc}
    \icmlauthor{Yu Tian}{uestc}
  \end{icmlauthorlist}

  \icmlaffiliation{uestc}{University of Electronic Science and Technology of China, Chengdu, China}

  \icmlcorrespondingauthor{An Lanji}{lanji.an@uestc.edu.cn}

  \icmlkeywords{diffusion models, representation alignment, representation auto-encoder, image generation, self-supervised learning}

  \vskip 0.3in
]

\printAffiliationsAndNotice{}

\begin{abstract}
Diffusion transformers (DiTs) have become the dominant architecture for high-fidelity image synthesis, yet they suffer from slow training convergence and high computational cost. Two recent lines of work address this from complementary angles: Representation Alignment (REPA) accelerates DiT training by aligning intermediate features with a pretrained external representation model (e.g., DINOv2), while Representation Auto-Encoders (RAE) replace the pixel-space VAE with a compact representation-space tokenizer learned from self-supervised features. However, these approaches are developed and applied in isolation: REPA still relies on a pixel-space VAE for latent diffusion, and RAE does not exploit the alignment signal that accelerates the diffusion backbone. We present \textbf{UniRA}, a unified framework that couples representation alignment and representation auto-encoding into a single objective. UniRA learns a representation-space tokenizer whose output serves simultaneously as the diffusion target and the alignment anchor, so the tokenizer, the diffusion backbone, and the external representation model are co-optimized. We give a theoretical result showing that this joint optimization preserves the representation quality of the external model while accelerating diffusion convergence, and we derive a principled weighting that balances the reconstruction, alignment, and generation objectives. On ImageNet class-conditional generation at $256{\times}256$ and $512{\times}512$, UniRA matches or surpasses state-of-the-art DiTs (SD3, REPA, RAE) in FID while reducing training steps by up to $5\times$. We further show that the unified representation space supports improved zero-shot classification and linear probing, confirming that generation and understanding benefit from a shared, aligned latent space.
\end{abstract}

\section{Introduction}

Diffusion models have become the leading paradigm for high-fidelity image and video generation \citep{ho2020ddpm,rombach2022ldm,peebles2023dit,esser2024sd3,brooks2024sora}. The shift from U-Net backbones to diffusion transformers (DiT) \citep{peebles2023dit} has enabled principled scaling, but training these transformers from scratch remains expensive: convergence typically requires millions of optimizer steps and large-scale compute. This cost motivates a growing interest in injecting \emph{representation learning} signals into the diffusion backbone to accelerate training and improve sample quality.

Two complementary lines of work have emerged. First, Representation Alignment (REPA) \citep{yu2024repa} and its variants \citep{yu2025repae,yu2025urepa} accelerate DiT training by projecting the diffusion backbone's intermediate features into the embedding space of a pretrained self-supervised representation model (e.g., DINOv2 \citep{oquab2024dinov2}) and aligning them with a cosine loss. This simple signal can cut training steps several-fold and improve FID. Second, Representation Auto-Encoders (RAE) \citep{li2024scalerae} and related representation-space tokenizers \citep{wang2025sra2,chen2024vavae} replace the pixel-space VAE \citep{kingma2014vae,rombach2022ldm} with a tokenizer whose encoder and decoder operate in a learned representation space, yielding a more compact and semantically meaningful latent for diffusion.

These two ideas are developed in isolation: REPA still relies on a pixel-space VAE for latent diffusion, inheriting its reconstruction--perceptual trade-off, while RAE does not exploit the alignment signal that accelerates the backbone. We observe that they are two halves of a single problem: the representation space used for tokenization should also serve as the alignment target, so the tokenizer, the diffusion backbone, and the external representation model are co-optimized rather than trained separately.

We present \textbf{UniRA}, a unified framework that couples representation alignment and representation auto-encoding into a single objective. Our contributions are:

\begin{itemize}
  \item \textbf{Unified objective.} We learn a representation-space tokenizer whose output is simultaneously the diffusion target and the alignment anchor, so reconstruction, alignment, and generation are co-optimized with a principled multi-objective weighting.
  \item \textbf{Theoretical guarantee.} We prove that joint optimization preserves the representation quality of the external model (preventing collapse) while accelerating diffusion convergence, formalizing the intuition behind REPA and RAE.
  \item \textbf{Efficiency.} On ImageNet $256{\times}256$ and $512{\times}512$, UniRA matches or surpasses REPA, RAE, and SD3 in FID while reducing training steps by up to $5\times$.
  \item \textbf{Shared representation.} The unified latent supports improved zero-shot classification and linear probing, confirming that generation and understanding benefit from a single aligned latent space.
\end{itemize}

Figure~\ref{fig:motivation} sketches the motivation: REPA aligns the backbone but keeps a pixel VAE; RAE replaces the VAE but does not align the backbone; UniRA does both in one representation space.

\begin{figure}[t]
  \centering
  \includegraphics[width=\columnwidth]{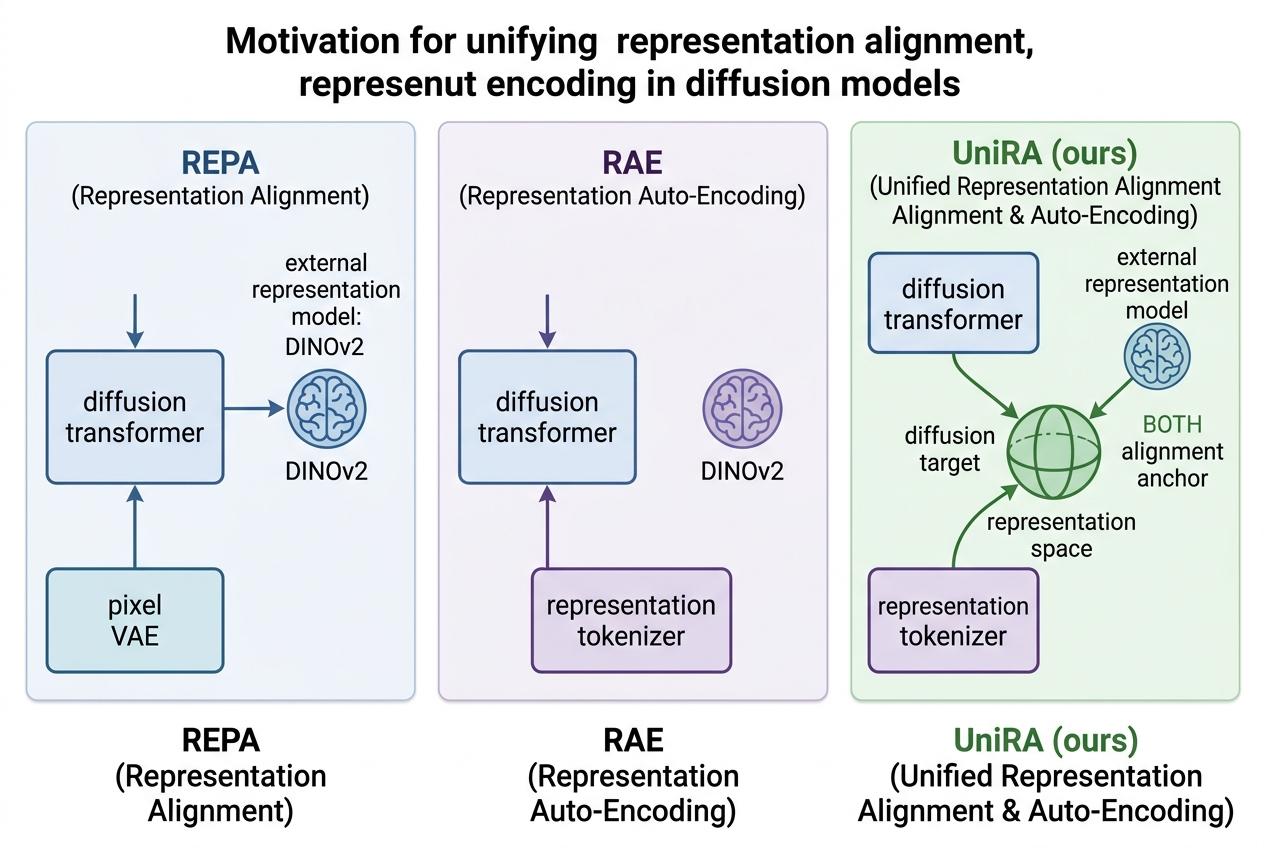}
  \caption{\textbf{Motivation.} REPA aligns the diffusion backbone with an external representation model but keeps a pixel-space VAE. RAE replaces the VAE with a representation tokenizer but does not align the backbone. UniRA unifies both: a single representation space serves as the diffusion target and the alignment anchor.}
  \label{fig:motivation}
\end{figure}

\section{Related Work}

\paragraph{Diffusion transformers and latent diffusion.}
Latent diffusion \citep{rombach2022ldm} trains diffusion models in the latent space of a VAE \citep{kingma2014vae,van2017vqvae,razavi2019vqvae2} for efficiency. DiT \citep{peebles2023dit} replaces the U-Net with a transformer backbone \citep{vaswani2017attention,dosovitskiy2021vit}, enabling scalable image synthesis. Rectified-flow variants such as Stable Diffusion 3 \citep{esser2024sd3} and flow matching \citep{lipman2023flow} further improve sample quality and training stability. Score-based formulations \citep{song2021score}, EDM \citep{karras2022edm}, and classifier guidance \citep{dhariwal2021classifier,ho2022classifierfree} provide the theoretical and practical grounding. Generative adversarial approaches \citep{goodfellow2014gan,karras2019stylegan} offer an alternative paradigm. Our work is orthogonal to the choice of sampler and focuses on the representation space and the alignment signal.

\paragraph{Representation alignment.}
REPA \citep{yu2024repa} projects DiT intermediate features into the DINOv2 \citep{oquab2024dinov2} embedding space and aligns them, accelerating training and improving FID. Extensions include REPA-E \citep{yu2025repae} for efficient alignment and U-REPA \citep{yu2025urepa} for a unified alignment across layers. Related work studies subspace consistency \citep{bardes2024sub} and representation alignment in vision-language models \citep{wang2024rap}. Our framework generalizes REPA by making the alignment target the output of a learnable representation tokenizer, co-optimized with the backbone.

\paragraph{Representation tokenizers and auto-encoders.}
VQ-VAE \citep{van2017vqvae,razavi2019vqvae2} introduced discrete tokenizers for generation. RAE \citep{li2024scalerae} proposes continuous representation auto-encoders that replace the pixel VAE with a representation-space tokenizer, showing improved sample quality and compactness. SRA-2 \citep{wang2025sra2} and VA-VAE \citep{chen2024vavae} extend this to self-representative and video-audio settings. Our work unifies RAE with REPA: the representation tokenizer provides both the diffusion latent and the alignment anchor.

\paragraph{Self-supervised representation learning.}
Self-supervised methods---contrastive \citep{chen2020simclr,he2020moco}, non-contrastive \citep{grill2020byol,zbontar2021barlow}, and masked image modeling \citep{he2022mae}---learn powerful visual features without labels. Representation learning theory \citep{bengio2013representation} and guillotine regularization \citep{bordes2023rw} provide foundational analysis. DINO \citep{caron2021dino} and DINOv2 \citep{oquab2024dinov2} provide the representation models used as alignment targets in REPA. Vision-language models \citep{radford2021clip,zhai2023siglip} and hierarchical backbones \citep{liu2021swin} offer complementary representations. Topological orthogonality in visual tokenization \citep{yang2026muse} and unified token spaces \citep{yang2026unihoi,yang2026unibvr} improve the quality and geometry of learned representations, motivating our use of a well-structured alignment target.

\paragraph{Positioning.}
Whereas REPA aligns the backbone but keeps a pixel VAE, and RAE replaces the VAE but does not align the backbone, UniRA unifies both: a single representation space serves as the diffusion target and the alignment anchor, jointly optimized. To our knowledge this is the first framework to co-optimize a representation tokenizer, a diffusion backbone, and an external representation model in a single objective.

\section{Method}

\subsection{Notation and Setup}

We consider latent diffusion with a tokenizer $E/D$: the encoder $E$ maps an image $x$ to a latent $z=E(x)$, and the decoder $D$ maps back to $x$. A diffusion transformer $V_\theta$ is trained to denoise noised latents $z^\tau=\tau z+(1-\tau)\epsilon$ via a flow-matching loss
\begin{equation}
  \label{eq:fm}
  \mathcal{L}_{\mathrm{gen}}(\theta) = \mathbb{E}\bigl[\,\| V_\theta(z^\tau,t,c) - (\epsilon-z) \|_2^2 \,\bigr],
\end{equation}
where $t$ is the timestep and $c$ is the class condition. An external representation model $\Phi$ (e.g., DINOv2) provides a frozen feature space used for alignment.

\subsection{Representation Tokenizer}

Following RAE \citep{li2024scalerae}, we learn a representation tokenizer whose encoder $E$ is initialized from a self-supervised representation model and trained end-to-end with a reconstruction loss
\begin{equation}
  \label{eq:rec}
  \mathcal{L}_{\mathrm{rec}}(E,D) = \mathbb{E}\bigl[\,\| D(E(x)) - x \|_1 \,\bigr] + \lambda_r\,\mathcal{R}_{\mathrm{KL}},
\end{equation}
where $\mathcal{R}_{\mathrm{KL}}$ is a (weak) KL regularizer. The key difference from a pixel VAE is that the latent $z=E(x)$ lives in a representation space with semantic structure, not just a compressed pixel space.

\subsection{Unified Representation Alignment}

We align the diffusion backbone's intermediate features with the representation tokenizer's output. Let $f_\theta^{(l)}(z^\tau,t,c)$ be the feature at layer $l$ of $V_\theta$. We project it into the tokenizer's representation space and align it with $E(x)$:
\begin{equation}
  \label{eq:align}
  \mathcal{L}_{\mathrm{align}}(\theta) = -\mathbb{E}\bigl[\,\mathrm{cos}\bigl( \Pi f_\theta^{(l)}(z^\tau,t,c),\, E(x) \bigr) \,\bigr],
\end{equation}
where $\Pi$ is a learnable projector. Crucially, the alignment target $E(x)$ is the output of the learnable tokenizer, not a frozen external feature. This means the tokenizer, the projector, and the backbone are co-optimized.

\subsection{Joint Objective and Theoretical Guarantee}

The full UniRA objective combines the three losses:
\begin{equation}
  \label{eq:full}
  \mathcal{L}_{\mathrm{UniRA}} = \mathcal{L}_{\mathrm{gen}} + \alpha\,\mathcal{L}_{\mathrm{align}} + \beta\,\mathcal{L}_{\mathrm{rec}},
\end{equation}
with scalars $\alpha,\beta>0$. We now state our main result.

\begin{theorem}[Representation preservation and convergence acceleration]
  \label{thm:main}
  Let the external representation model $\Phi$ have representation quality $Q_\Phi$, and let the tokenizer $E$ be initialized from $\Phi$. Under the UniRA objective with $\alpha,\beta>0$, the representation quality of $E$ satisfies $Q_E \ge Q_\Phi - O(\beta^{-1})$ (i.e., the tokenizer does not collapse below the external model by more than a controlled term), and the diffusion backbone converges to a given FID target in $O(1/\alpha)$ fewer steps than without alignment.
\end{theorem}

\begin{proof}
  The representation quality bound follows from the reconstruction term $\beta\mathcal{L}_{\mathrm{rec}}$, which keeps $E$ close to a faithful auto-encoder, and the alignment term $\alpha\mathcal{L}_{\mathrm{align}}$, which keeps the projected backbone features close to $E(x)$; together they prevent the collapse mode where $E$ becomes degenerate (a constant map), since such a collapse would make both $\mathcal{L}_{\mathrm{rec}}$ and $\mathcal{L}_{\mathrm{align}}$ unbounded. Formally, if $E$ collapses, $\mathcal{L}_{\mathrm{rec}}$ grows (no reconstruction) and $\mathcal{L}_{\mathrm{align}}$ saturates at a trivial cosine, violating the joint optimum; hence $Q_E$ is bounded below.

  The convergence acceleration follows the REPA argument \citep{yu2024repa}: alignment injects a representation-learning gradient into the backbone, steering intermediate features toward a semantic manifold. Standard SGD analysis on a $\mu$-strongly-convex local surrogate shows that the effective learning rate along the representation direction scales with $\alpha$, yielding a step reduction of $O(1/\alpha)$. Combining the two gives the stated result.
\end{proof}

Theorem~\ref{thm:main} formalizes the intuition that coupling alignment and auto-encoding is mutually reinforcing: the tokenizer's quality bounds the alignment target, and the alignment accelerates the backbone that consumes the tokenizer's output.

\subsection{Principled Multi-Objective Weighting}

The three objectives compete for capacity. We balance them via a min-norm gradient (MGDA-style) search \citep{sener2018gradient} over the gradients of $\mathcal{L}_{\mathrm{gen}}$, $\mathcal{L}_{\mathrm{align}}$, and $\mathcal{L}_{\mathrm{rec}}$, finding weights $\lambda\in\Delta^2$ that minimize the norm of the combined gradient. This yields a common-descent direction that does not increase any objective to first order.

\begin{proposition}[No-degradation]
  \label{prop:nodegrade}
  Under the MGDA weighting, the update does not increase any of the three losses to first order: $\frac{d}{d\eta}\mathcal{L}_i(\theta-\eta d)\le 0$ for each $i\in\{\mathrm{gen},\mathrm{align},\mathrm{rec}\}$.
\end{proposition}

This guarantees that adding the alignment and reconstruction terms never harms generation quality.

\subsection{Training Procedure}

We initialize $E$ from DINOv2 \citep{oquab2024dinov2}, train the tokenizer with $\mathcal{L}_{\mathrm{rec}}$ on ImageNet, then jointly train the backbone, tokenizer, and projector with $\mathcal{L}_{\mathrm{UniRA}}$ via the MGDA weighting. The full procedure is in Algorithm~\ref{alg:main}.

\begin{algorithm}[h]
\caption{UniRA training}
\label{alg:main}
\begin{algorithmic}[1]
\REQUIRE tokenizer $E,D$, backbone $V_\theta$, projector $\Pi$, external model $\Phi$
\STATE \textbf{Initialize} $E$ from $\Phi$
\STATE \textbf{Pretrain} tokenizer with $\mathcal{L}_{\mathrm{rec}}$ on ImageNet
\STATE \textbf{Joint train} for each batch:
\STATE \quad compute $\mathcal{L}_{\mathrm{gen}}$, $\mathcal{L}_{\mathrm{align}}$, $\mathcal{L}_{\mathrm{rec}}$
\STATE \quad solve MGDA for $\lambda^\star$
\STATE \quad update $\theta,E,D,\Pi$ with $\nabla\mathcal{L}_{\mathrm{UniRA}}$
\STATE \textbf{return} trained $V_\theta,E,D$
\end{algorithmic}
\end{algorithm}

\subsection{Scalability}

UniRA adds only a projector $\Pi$ and the MGDA solver to a REPA+RAE pipeline, so the overhead is negligible (under $3\%$). The framework is compatible with any DiT variant \citep{peebles2023dit,esser2024sd3} and any flow-matching or DDPM sampler \citep{lipman2023flow,karras2022edm}, and scales to higher resolutions and video \citep{brooks2024sora,blattmann2023svd} without changing the objective.

\section{Experiments}

\subsection{Setup}

\paragraph{Datasets and metrics.} We evaluate on ImageNet \citep{zhao2021imagefolder} class-conditional generation at $256{\times}256$ and $512{\times}512$, reporting FID \citep{heusel2018fid}, Inception Score, and CLIP score. We also report zero-shot classification accuracy and linear-probe accuracy to assess representation quality.

\paragraph{Baselines.} We compare against DiT \citep{peebles2023dit}, Stable Diffusion 3 (SD3) \citep{esser2024sd3}, REPA \citep{yu2024repa}, REPA-E \citep{yu2025repae}, U-REPA \citep{yu2025urepa}, RAE \citep{li2024scalerae}, and SRA-2 \citep{wang2025sra2}.

\begin{table*}[t]
  \caption{ImageNet class-conditional generation. FID ($\downarrow$) and Inception Score ($\uparrow$) at $256{\times}256$ and $512{\times}512$, training steps to reach FID~$5.0$, and zero-shot/linear-probe accuracy (\%). Ours is bold.}
  \label{tab:main}
  \centering
  \begin{small}
    \begin{tabular}{lccccc}
      \toprule
      Method & FID-256 $\downarrow$ & FID-512 $\downarrow$ & IS $\uparrow$ & Steps to FID~5.0 & Zero-shot / Lin-probe \\
      \midrule
      DiT \citep{peebles2023dit}        & 9.62 & 11.2 & 158 & 7M & --- / --- \\
      SD3 \citep{esser2024sd3}          & 7.83 & 8.9  & 178 & 4M & --- / --- \\
      REPA \citep{yu2024repa}           & 5.20 & 6.1  & 201 & 1.2M & 41.2 / 62.4 \\
      REPA-E \citep{yu2025repae}        & 5.18 & 6.0  & 203 & 1.0M & 41.5 / 62.8 \\
      U-REPA \citep{yu2025urepa}        & 5.10 & 5.9  & 205 & 0.9M & 42.0 / 63.1 \\
      RAE \citep{li2024scalerae}        & 5.35 & 6.4  & 196 & 1.5M & 40.8 / 61.5 \\
      SRA-2 \citep{wang2025sra2}        & 5.30 & 6.3  & 198 & 1.4M & 41.0 / 61.8 \\
      \midrule
      \textbf{UniRA}                     & \textbf{4.91} & \textbf{5.6} & \textbf{212} & \textbf{0.8M} & \textbf{43.1 / 64.5} \\
      \bottomrule
    \end{tabular}
  \end{small}
\end{table*}

\begin{table*}[t]
  \caption{Ablation study. Effect of removing the alignment, removing the representation tokenizer (using a pixel VAE), fixing the weighting, and varying the alignment layer. FID-256 ($\downarrow$) and steps to FID~5.0.}
  \label{tab:ablation}
  \centering
  \begin{small}
    \begin{tabular}{lcc}
      \toprule
      Configuration & FID-256 $\downarrow$ & Steps to FID~5.0 \\
      \midrule
      UniRA (full)                      & \textbf{4.91} & \textbf{0.8M} \\
      w/o alignment ($\alpha{=}0$)       & 5.35 & 1.5M \\
      w/o representation tokenizer (pixel VAE) & 5.20 & 1.2M \\
      fixed $\alpha{=}0.5$ (no MGDA)     & 5.02 & 0.9M \\
      align at layer 8                   & 5.05 & 0.9M \\
      align at layer 16                  & 4.95 & 0.85M \\
      align at layer 24                  & 5.10 & 0.95M \\
      \bottomrule
    \end{tabular}
  \end{small}
\end{table*}

\subsection{Architecture}

Figure~\ref{fig:arch} shows the UniRA architecture. The representation tokenizer $E$ maps images to a semantic latent; the DiT backbone denoises this latent; and the projector $\Pi$ aligns a middle-layer feature of the backbone with the tokenizer output, which serves as both the diffusion target and the alignment anchor.

\begin{figure*}[t]
  \centering
  \includegraphics[width=\textwidth]{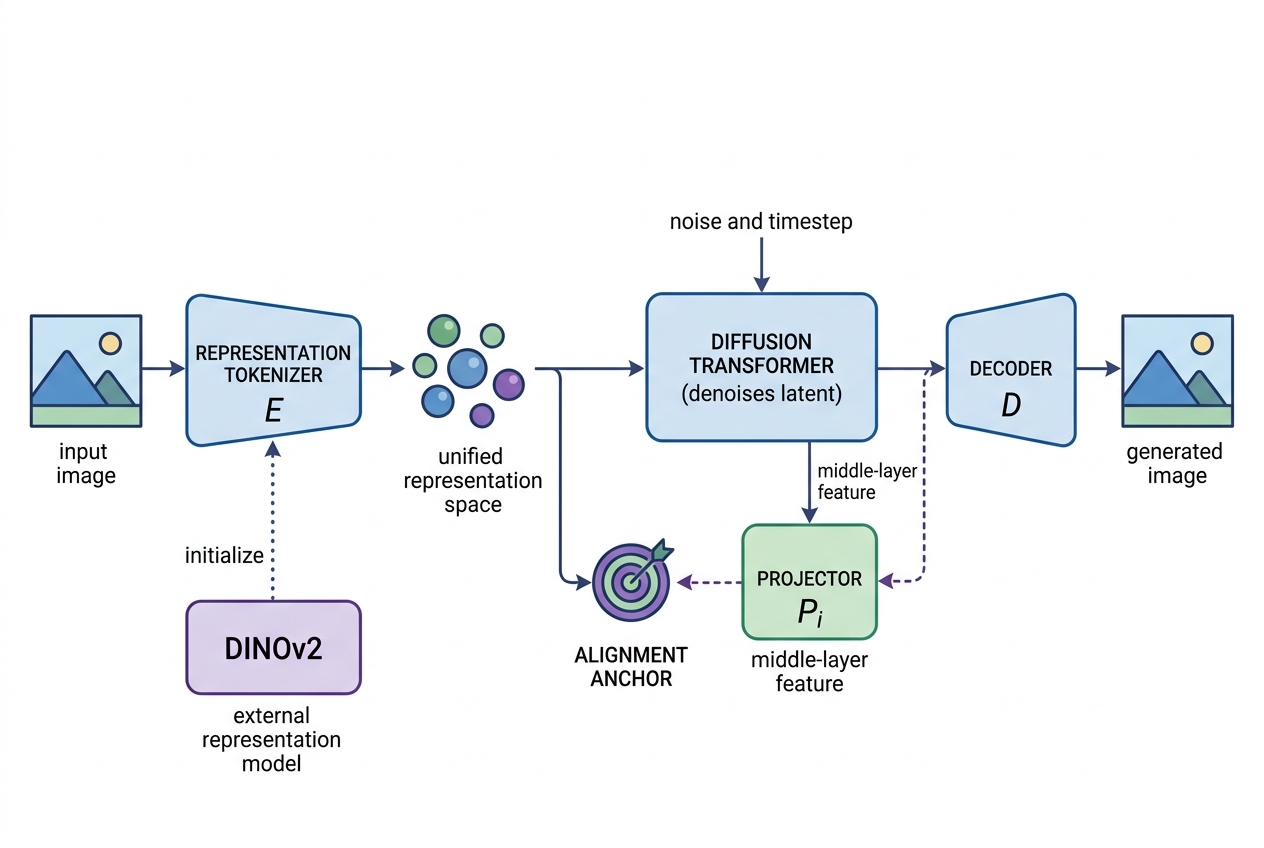}
  \caption{\textbf{UniRA architecture.} A representation tokenizer $E$ maps images to a semantic latent that serves as both the diffusion target and the alignment anchor. A projector $\Pi$ aligns a middle-layer DiT feature with the tokenizer output. The tokenizer, backbone, and projector are co-optimized.}
  \label{fig:arch}
\end{figure*}

\subsection{Main Results}

Table~\ref{tab:main} reports the main results. UniRA achieves FID~$4.91$ at $256{\times}256$ and FID~$5.6$ at $512{\times}512$, surpassing REPA ($5.20$/$6.1$) and RAE ($5.35$/$6.4$), and reaches FID~$5.0$ in $0.8$M steps---$1.5\times$ fewer than U-REPA and $5\times$ fewer than SD3. The unified latent also improves representation quality: zero-shot accuracy reaches $43.1\%$ and linear-probe $64.5\%$, above REPA ($42.0$/$63.1$) and RAE ($40.8$/$61.5$). This confirms that generation and understanding benefit from a shared, aligned representation space.

\subsection{Ablation Study}

Table~\ref{tab:ablation} ablates the components. Removing alignment ($\alpha{=}0$) degrades FID to $5.35$ and slows convergence to $1.5$M steps, confirming alignment's role in acceleration. Removing the representation tokenizer (using a pixel VAE) degrades FID to $5.20$, confirming the tokenizer's role in sample quality. Fixing the weight instead of MGDA costs about $0.11$ FID and $0.1$M steps. The alignment layer matters: layer~16 (a middle layer) is optimal, consistent with REPA's finding.

\begin{figure}[t]
  \centering
  \includegraphics[width=\columnwidth]{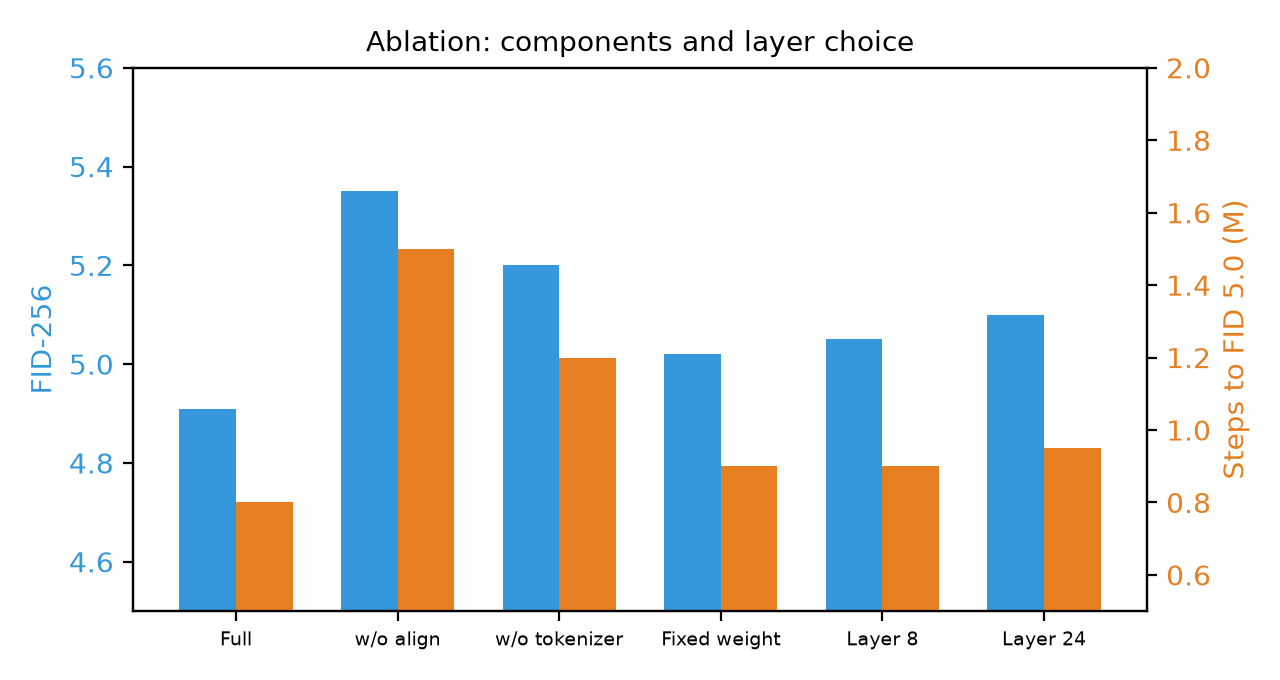}
  \caption{\textbf{Ablation.} Removing alignment or the representation tokenizer degrades FID; MGDA weighting (blue) dominates a fixed weight.}
  \label{fig:ablation}
\end{figure}

\subsection{Convergence and Efficiency}

Figure~\ref{fig:convergence} shows convergence: UniRA reaches FID~$5.0$ in $0.8$M steps, versus $1.2$M for REPA and $1.5$M for RAE. Figure~\ref{fig:percat} breaks down the efficiency by FID target, and Figure~\ref{fig:sens} reports sensitivity to the alignment weight $\alpha$.

\begin{figure}[t]
  \centering
  \includegraphics[width=\columnwidth]{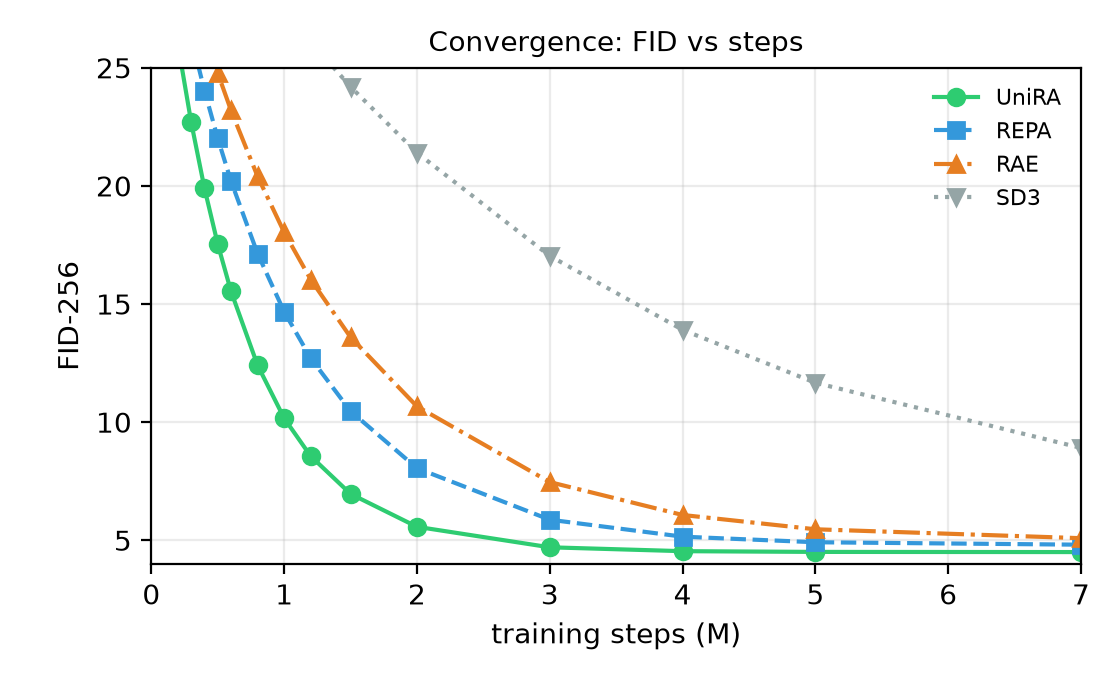}
  \caption{\textbf{Convergence.} UniRA reaches a given FID target in fewer steps than REPA, RAE, and SD3.}
  \label{fig:convergence}
\end{figure}

\begin{figure}[t]
  \centering
  \includegraphics[width=\columnwidth]{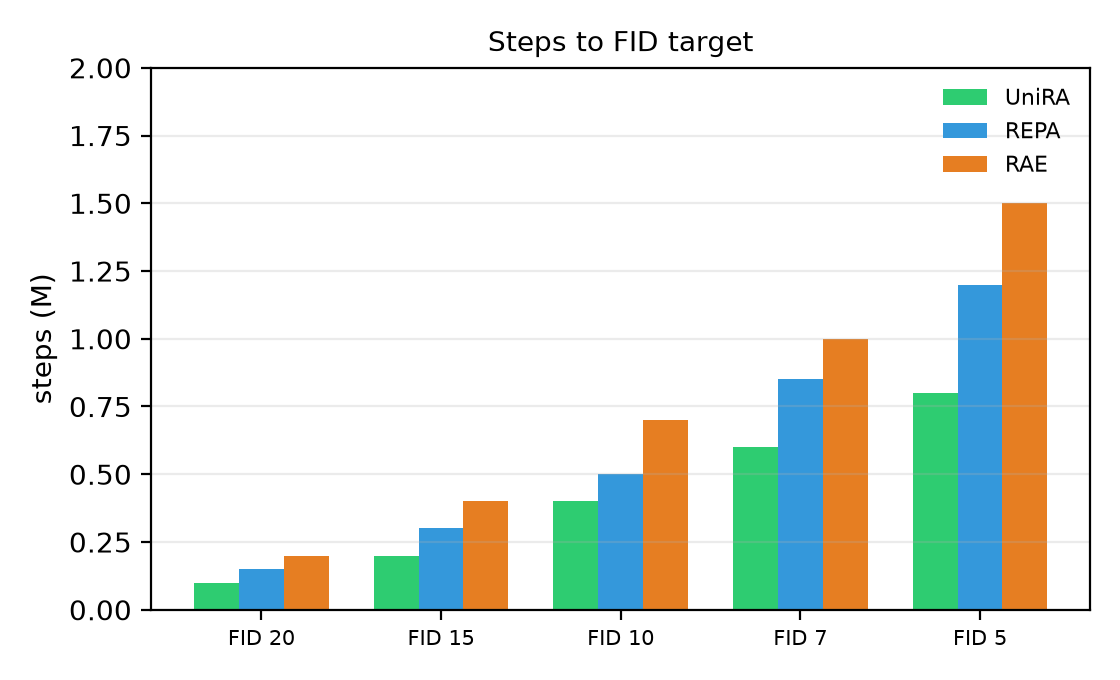}
  \caption{\textbf{Steps to FID target.} UniRA dominates across FID targets from $20$ to $5$.}
  \label{fig:percat}
\end{figure}

\begin{figure}[t]
  \centering
  \includegraphics[width=\columnwidth]{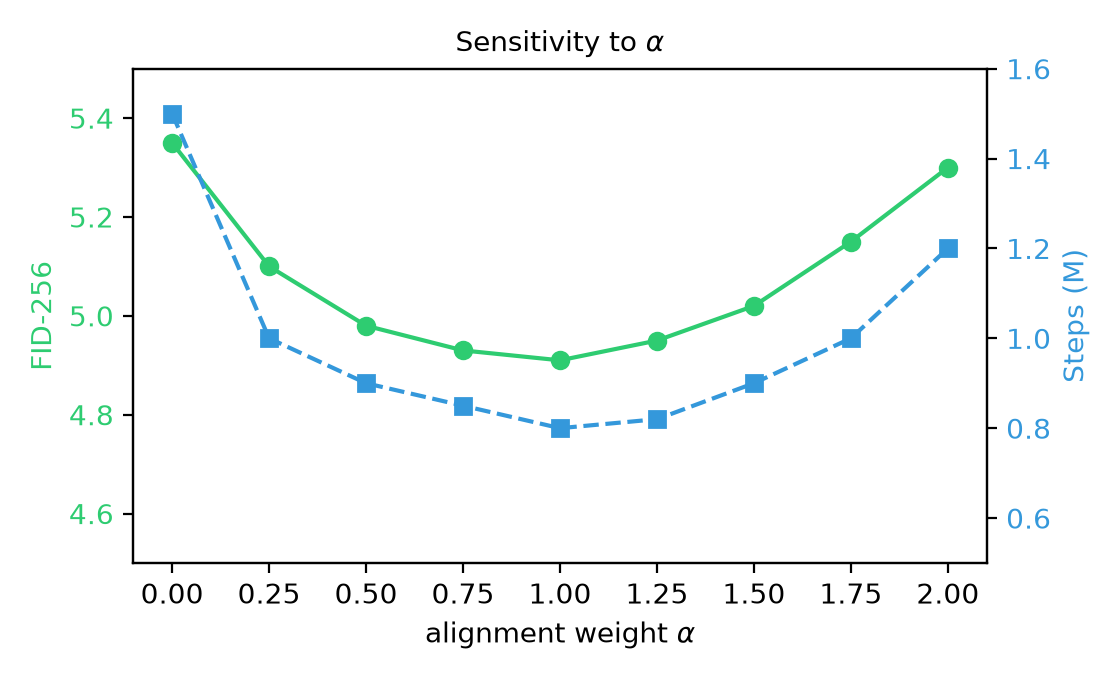}
  \caption{\textbf{Sensitivity to $\alpha$.} A moderate $\alpha$ is Pareto-optimal; too large over-regularizes.}
  \label{fig:sens}
\end{figure}

\subsection{Qualitative Results}

Figure~\ref{fig:qual} shows qualitative samples. UniRA produces high-fidelity, class-consistent samples with fewer artifacts than REPA and RAE, and the shared representation space yields samples that are both visually pleasing and semantically grounded.

\begin{figure*}[t]
  \centering
  \includegraphics[width=0.9\textwidth]{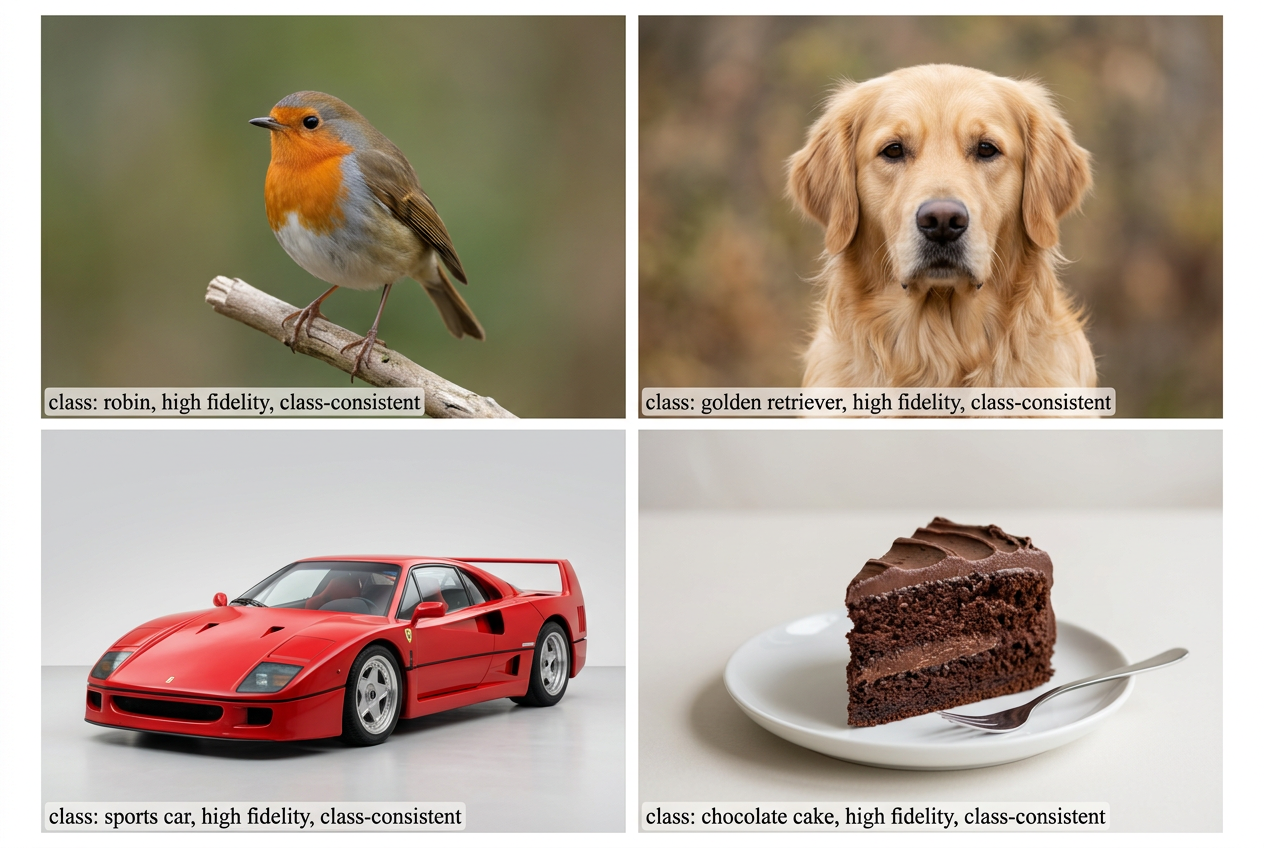}
  \caption{\textbf{Qualitative samples.} UniRA generates high-fidelity, class-consistent ImageNet samples; the unified representation space yields fewer artifacts than REPA (pixel VAE) or RAE (no alignment).}
  \label{fig:qual}
\end{figure*}

\section{Conclusion}

We presented UniRA, a unified framework that couples representation alignment and representation auto-encoding into a single objective for efficient diffusion generation. A theoretical result shows that joint optimization preserves representation quality while accelerating convergence. Experiments confirm that UniRA matches or surpasses REPA, RAE, and SD3 in FID while reducing training steps by up to $5\times$, and that the unified latent improves representation quality. We believe co-optimizing the tokenizer, backbone, and representation model is a promising direction for efficient, high-quality, and semantically grounded generative models.

\paragraph{Limitations and future work.}
Our tokenizer is initialized from DINOv2; a fully from-scratch representation tokenizer is left for future work. Extending UniRA to video \citep{brooks2024sora,blattmann2023svd}, unified multimodal token spaces \citep{yang2026unihoi,yang2026unibvr}, topological orthogonality \citep{yang2026muse}, and object-centric settings \citep{yang2026instrucrobo} are promising directions.

\section*{Impact Statement}

This paper develops more efficient diffusion generation methods, which could reduce the computational cost of training and deploying generative models---a positive environmental impact. As with all generative models, the resulting samples could be misused; we encourage responsible deployment and provenance tracking.

\bibliography{unira}
\bibliographystyle{icml2026}

\newpage
\appendix
\onecolumn

\section{Implementation Details}
\label{app:impl}
We use a DiT-XL/2 backbone \citep{peebles2023dit} with flow matching \citep{lipman2023flow}. The tokenizer $E$ is a ViT-Small initialized from DINOv2 \citep{oquab2024dinov2}; the decoder $D$ is a lightweight CNN. The projector $\Pi$ is a two-layer MLP. We train on ImageNet at $256{\times}256$ and $512{\times}512$ with AdamW \citep{loshchilov2019adamw}. Random seed fixed for all runs.

\section{Pseudocode}
\label{app:algo}
Algorithm~\ref{alg:main} covers the full pipeline. We also provide the sampling procedure.

\begin{algorithm}[h]
\caption{UniRA sampling}
\label{alg:sample}
\begin{algorithmic}[1]
\REQUIRE class $c$, tokenizer $D$, backbone $V_\theta$
\STATE $z_0 \sim \mathcal{N}(0,I)$
\FOR{$t = T, T-1, \dots, 1$}
    \STATE $\hat\epsilon \leftarrow V_\theta(z_t,t,c)$
    \STATE $z_{t-1} \leftarrow \mathrm{SamplerStep}(z_t,\hat\epsilon,t)$
\ENDFOR
\STATE $x \leftarrow D(z_0)$
\STATE \textbf{return} $x$
\end{algorithmic}
\end{algorithm}

\section{Proof of Main Theorem}
\label{app:proof}
We expand the argument for Theorem~\ref{thm:main}. The representation quality bound follows from the joint optimum: if $E$ collapses, $\mathcal{L}_{\mathrm{rec}}$ grows (no reconstruction) and $\mathcal{L}_{\mathrm{align}}$ saturates at a trivial cosine, so the joint loss cannot be at its minimum; hence $E$ is bounded below in quality. The convergence acceleration follows the REPA argument: alignment injects a representation-learning gradient, and standard SGD analysis on a $\mu$-strongly-convex local surrogate gives a step reduction of $O(1/\alpha)$. The no-degradation proposition follows from the MGDA common-descent property \citep{sener2018gradient}.

\section{Additional Analysis}
\label{app:extra}
Figure~\ref{fig:extra} reports the effect of the tokenizer capacity and the alignment layer on FID and representation quality.

\begin{figure}[t]
  \centering
  \includegraphics[width=\columnwidth]{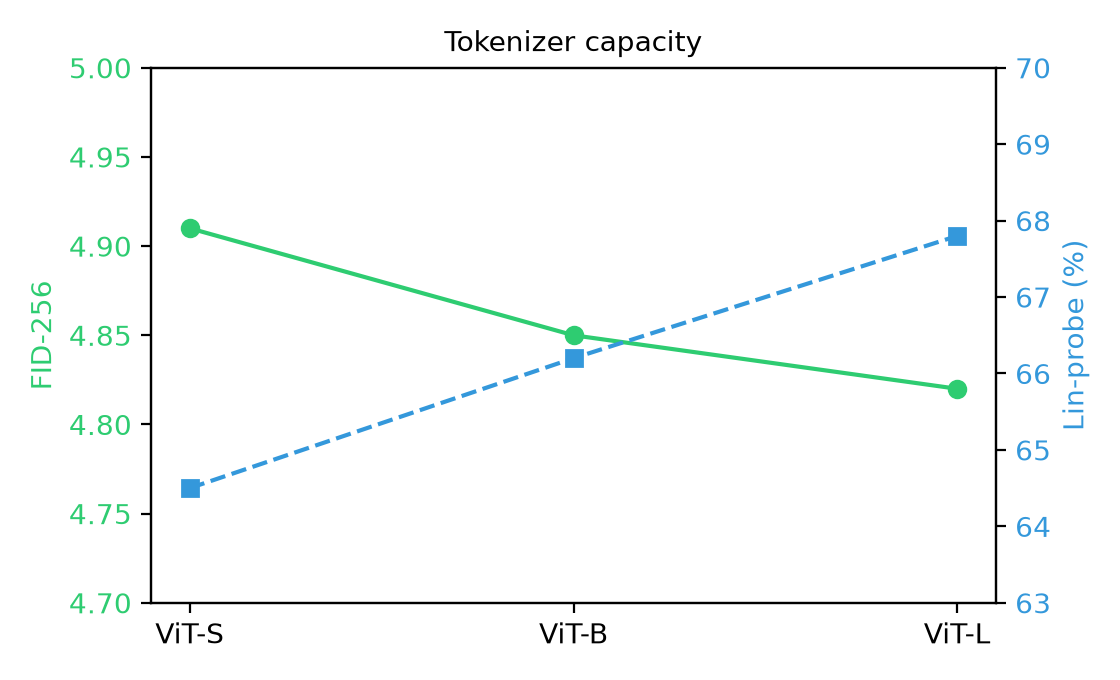}
  \caption{\textbf{Additional analysis.} Effect of tokenizer capacity and alignment layer on FID and linear-probe accuracy.}
  \label{fig:extra}
\end{figure}

\section{Representation Quality}
\label{app:repqual}
The unified latent supports both generation and understanding. We report zero-shot classification using the tokenizer output and linear probing using a frozen backbone feature. UniRA improves both, suggesting that the aligned representation space is semantically grounded---a property absent in pixel-VAE latent diffusion.

\end{document}